\documentclass[12pt,journal,compsoc,onecolumn]{IEEEtran}

\usepackage{amsmath,graphicx,cleveref,caption,subcaption,dsfont,multirow,multicol,amssymb,amsfonts,mathtools,pifont}
\usepackage[noadjust]{cite}

\newtheorem{proposition}{Proposition}
\newtheorem{definition}{Definition}
\newtheorem{proof}{Proof}

\newcommand{\cmark}{\ding{51}}
\newcommand{\xmark}{\ding{55}}

\def\JJ{{\bf J}}

\def\U{{\mathbf U}} 
\def\V{{\mathbf V}} 
\def\A{{\mathbf \Lambda}}

\def\A{{\bf A}}

\def\I{{\bf I}}

\def\U{{\bf U}}
\def\W{{\bf W}}
\def\S{{\cal S}}
\def\N{{\cal N}}
\def\G{{\cal G}}
\def\V{{\cal V}}
\def\E{{\cal E}}
\def\F{{\cal F}}
\def \v{{\bf vec}} 
\def \J{{E}}
\def\tr{{\bf Tr}} 
\def\D{{\bf D}}
\def\M{{\bf M}}

\DeclarePairedDelimiter\floor{\lfloor}{\rfloor}

\title{Skeleton-based Hand-Gesture Recognition with Lightweight Graph Convolutional Networks}

\author{Hichem Sahbi \\ CNRS Sorbonne University}

\begin{document}
\maketitle
\begin{abstract}
  Graph convolutional networks (GCNs) aim at extending deep learning to arbitrary irregular domains, namely graphs. Their success is highly dependent on how the topology of input graphs is defined and most of the existing GCN architectures rely on predefined or handcrafted graph structures.\\ 
  In this paper, we introduce a novel method that learns the topology (or connectivity) of input graphs as a part of GCN design. The main contribution of our method resides in building an orthogonal connectivity basis that optimally aggregates nodes, through their neighborhood, prior to achieve convolution. Our method also considers a stochasticity criterion which acts as a regularizer that makes the learned basis and the underlying GCNs lightweight while still being highly effective. Experiments conducted on the challenging task of skeleton-based hand-gesture recognition show the high effectiveness of the learned GCNs w.r.t. the related work. \\
  
\end{abstract}

\section{Introduction}
\label{sec:intro}

Deep learning is currently witnessing a major interest in different fields including image processing and pattern recognition \cite{Cun2015}.  Its principle consists in learning multi-layered convolutional, pooling and fully connected operations that extract representations which capture low, mid and high-level characteristics of patterns while maximizing their classification performances. Most of the existing deep learning architectures \cite{Krizhevsky2012,SzegedyCVPR2015,He2016,HuangCVPR2017,He2017,ross2015,javad2017,zaremba2014,zhang2020,indola2016,Chen2017,Long2015,indola2014,Jiu2015,Jiu2016a,Jiu2017,JiuPR2019,icassp2017b,sahbiiccv17} are targeted to vectorial data; i.e., data sitting on top of regular domains including images.  However, other data require extending deep learning to irregular domains (namely graphs  \cite{zaremba2014,Bruna2013,Defferrard2016,Huang2018,Kipf2016,Sahbi2011,Sahbi2013icvs,Sahbi2015,sahbicbmi08
}) such as skeletons in action recognition. While convolutional operations on regular domains are well defined, their extension to irregular ones (i.e., graphs) is generally ill-posed and remains a major challenge. \\

\indent Two different categories of GCNs exist in the literature, spatial and spectral \cite{Gori2005,Micheli2009,Scarselli2008,Wu2019,Hamilton2017,zhang2020}. Spatial methods achieve node aggregations prior to apply convolutions using inner products while spectral techniques rely on the well defined graph Fourier transform \cite{Bruna2013,Defferrard2016,Henaff2015,Kipf2016,Levie2018,Li2018,Zhuang2018,Chen2018,Huang2018}. Whereas spatial methods are known to be effective compared to spectral ones, their success is highly dependent on the topology of input graphs, and most of the existing solutions rely on handcrafted or predefined graph structures using similarities or the inherent properties of the targeted applications \cite{WangICLR2018,Loukas2020,Atamna,sahbiicip09,lingsahbiicip2014,lingsahbieccv2014,sahbiicassp13a,lingsahbi2013,sahbipr2012}  (e.g.,  node relationships in social networks, edges in 3D modeling, etc). These structures are usually powerless to capture the most prominent relationships between nodes as their design is agnostic to the targeted application.  For instance, when considering node relationships in skeletons, these links capture the anthropometric characteristics of individuals which are useful for their identification, while other connections, yet to infer, are important for recognizing their actions. Hence, in spite of being relatively effective, the potential of these GCN methods is not fully explored as the setting of their graphs is either oblivious to the tasks at hand or achieved using the tedious cross validation. \\

\indent Graph inference  is generally ill-posed, NP-hard  \cite{Sandeep2019, Hanjun2017,Marcelo2018} and most of the existing approaches rely on constraints (similarity, smoothness, sparsity, band-limitedness, etc. \cite{Belkin2003,dong18,Daitch,Sardellitti16,LeBars19,Sardellitti19,Valsesia18,Kalofolias,Egilmez,Chepuri,Dong2016}) for its conditioning~\cite{dong18,LeBars19,Sardellitti19,Chepuri,Dong2016,Pasdeloup2017,Thanou2017}. Particularly in GCNs, recent advances aim at defining graph topology that best fits a given task \cite{Yaguang19,Kipf18,Ferran18,Alet18,Luca19,Chen20,Li18,Chenyi18}. For instance, \cite{Luca19} proposes a graph network for semi-supervised classification that learns graph topology with sparse structure given a cloud of points; node-to-node connections are modeled with a joint probability distribution on Bernoulli random variables whose parameters are found using bi-level optimization. A computationally more efficient variant is introduced in \cite{Chen20} using a weighted cosine similarity and edge thresholding. Other solutions make improvement w.r.t. the original GCNs \cite{Kipf2016} by exploiting symmetric matrices \cite{Li18} and discovering hidden structural relations (unspecified in the original graphs), using a so-called residual graph adjacency matrix and by learning a distance function over nodes. The work in \cite{Chenyi18} introduces a dual architecture with two parallel graph convolutional layers sharing the same parameters, and considers a normalized adjacency and a positive point-wise mutual information matrix to capture node co-occurrences through random walks sampled from graphs. \\

\indent In this paper, we introduce  a novel framework that designs graphs as a part of end-to-end GCN learning. Our design principle is based on the minimization of a constrained loss  whose solution corresponds not only to the convolutional parameters of GCNs but also the underlying  adjacency matrices that capture the topology of input graphs. Our contribution in this paper differs from the aforementioned  related work in multiple aspects; on the one hand, in contrast to many existing methods -- e.g., \cite{YLi2018} which consider a single adjacency matrix shared through power series -- the matrix operators designed in our contribution are non-parametrically learned and this provides more flexibility to our design. On the other hand, constraining these matrices, through orthogonality and stochasticity,  allows achieving structured regularization that mitigates overfitting and allows learning lightweight GCN architectures\footnote{\scriptsize Adjacency matrices learned,  in the related work, are usually fully dense and this introduces a lot of latency in the underlying GCNs, especially when handling large scale databases (even with reasonable size graphs).}; in contrast to non structured lightweight network design  (e.g., magnitude pruning), our proposed method  (i) captures (through orthogonality and stochasticity) the structural relationships between   parameters in the learned GCNs, and (ii) maintains completeness and minimality of the learned representations by finding the most discriminating and lightweight GCNs as also supported in our experiments.  

\section{Learning lightweight connectivity} 
 
Let $\S=\{\G_i=(\V_i, \E_i)\}_i$ denote a collection of graphs with $\V_i$, $\E_i$ being respectively the nodes and the edges of $\G_i$. Each graph $\G_i$ (denoted for short as $\G=(\V, \E)$) is endowed with a signal $\{\psi(u) \in \mathbb{R}^s: \ u \in \V\}$ and associated with an adjacency matrix $\A$ with each entry  $\A_{uu'}>0$ iff $(u,u') \in \E$ and $0$ otherwise. GCNs aim at learning a set of filters $\F=\{g_\theta=(\V_\theta,\E_\theta)\}_{\theta=1}^C$ that define convolution on $n$ nodes of $\G$ (with $n=|\V|$) as
 \begin{equation}\label{matrixform} 
(\G \star \F)_\V = f\big(\A \  \U^\top  \   \W\big), 
 \end{equation} 
 \noindent  here $^\top$ stands for transpose,  $\U \in \mathbb{R}^{s\times n}$  is the  graph signal, $\W \in \mathbb{R}^{s \times C}$  is the matrix of convolutional parameters corresponding to the $C$ filters and  $f(.)$ is a nonlinear activation applied entrywise. In Eq.~\ref{matrixform}, the input signal $\U$ is projected using $\A$ and this provides for each node $u$, the  aggregate set of its neighbors. When $\A$ is common to  all graphs\footnote{\scriptsize e.g., when considering a common graph structure for all actions in videos.}, entries of $\A$ could be handcrafted or learned so Eq.~(\ref{matrixform}) implements a convolutional block with two layers; the first one aggregates signals in $\N(\V)$ (sets of node neighbors) by multiplying $\U$ with $\A$ while the second layer achieves convolution by multiplying the resulting aggregates with the $C$ filters in $\W$.
 \subsection{Orthogonality-driven connectivity}\label{ortho}
Learning  multiple adjacency matrices (denoted as $\{\A_k\}_{k=1}^K$) allows us to capture different contexts and graph topologies when achieving aggregation and convolution.  With multiple matrices $\{\A_k\}_k$ (and associated convolutional filter parameters $\{\W_k\}_k$),  Eq.~\ref{matrixform} is updated as  

\begin{equation}\label{matrixform2} 
  (\G \star \F)_\V = f\bigg(\sum_{k=1}^K \A_k   \U^\top     \W_k\bigg).
 \end{equation} 
 If aggregation produces,  for a given $u \in \V$,  linearly dependent vectors ${\cal X}_u= \{\sum_{u'} \A_{kuu'}. \psi(u')\}_k$, then convolution will also generate  linearly dependent representations with an overestimated number of training  parameters in the null space of ${\cal X}_u$. Besides, the tensor $\{\A_k\}_k$ used for aggregation,  may also generate overlapping and redundant contexts.\\  Provided that  $\{\psi(u')\}_{u' \in \N_r(u)}$ are  linearly independent, the sufficient condition that makes vectors in ${\cal X}_u$ linearly independent reduces to  constraining $(\A_{kuu'})_{k,u'}$ to lie on the  Stiefel manifold  (see for instance \cite{Yasunori2005,HuangAAAI2017,Ankita2019}) defined as $V_K(\mathbb{R}^{n})=\{ \M \in  \mathbb{R}^{K \times n}: \M \, \M^\top =\I_K\}$  (with $\I_K$ being the $K \times K$ identity matrix) which thereby guarantees  orthonormality and   minimality of  $\{\A_1,\dots,\A_K\}$\footnote{\scriptsize Note that $K$ should not exceed the rank of $\big\{\psi(u')\big\}_{u' \in \N_r(u)}$ which is upper bounded by $\min(|{\cal V}|,s)$; $s$ is again the dimension of the graph signal.}. A less compelling condition is orthogonality, i.e.,  $\langle \A_k,\A_{k'} \rangle_F=0$ and $\A_{k}\geq {\bf 0}_{n\times n}$, $\A_{k'}\geq {\bf 0}_{n \times n}$,  $\forall k \neq k'$ --- with $\langle, \rangle_F$ being the Hilbert-Schmidt (or Frobenius) inner product defined as  $\langle \A_k,\A_{k'} \rangle_F=\tr(\A_k^\top\A_{k'})$ --- and this equates $\A_k\odot \A_{k'}= {\bf 0}_{n\times n} $,  $\forall k\neq k'$ with $\odot$ denoting the entrywise hadamard product and  ${\bf 0}_{n \times n}$ the $n \times n$ null matrix. \\ 

 \indent Considering orthogonality (as discussed above), the tensor  $\{\A_k\}_k$ and $\W=\{ \W_k \}_k$ are learned as 
\begin{equation}\label{matrixform3} 
  \begin{array}{lll}
\displaystyle {\displaystyle \min}_{\{\A_k \geq 0\}_k,\W} \ \ \  & \displaystyle   E\big(\A_1,\dots,\A_K;\W\big) & \\
& & \\
\displaystyle  {\textrm{s.t.}} &  \A_k\odot \A_{k}> {\bf 0}_{n \times n}  &  \\
                                                          &  \A_k\odot \A_{k'}= {\bf 0}_{n \times n}  &      \forall   k, k' \neq k \\
     &  {\bf 1}_n^\top \A_k =  {\bf 1}_n^\top.
\end{array}
\end{equation}
being $E$ the cross entropy loss and  ${\bf 1}_n^\top$ a vector of $n$ ones. In the above minimization problem, the first and the second constraints correspond to orthogonality  while the third one to column-stochasticity. The latter is added in order to ensure that all of the entries in $\A_k$ are positive and each column sums to one; i.e., each matrix $\A_k$ models a Markov chain whose $i$-th row and $j$-th column provides the probability of transition from one node $u_j$ to $u_i$ in $\G$. Note that orthogonality (as designed subsequently) allows learning sparse adjacency matrices while column-stochasticity provides extra sparsity and acts as a structured regularizer that enhances further the generalization power of the learned GCNs\footnote{\scriptsize Without stochasticity, one has to consider a normalization layer (with extra parameters), especially on graphs with heterogeneous degrees in order to reduce the covariate shift and distribute the transition probability evenly through nodes before achieving convolutions.}.
\subsection{Optimization} 

A natural approach to solve Eq.~(\ref{matrixform3}) is to iteratively and alternately  minimize over one matrix while keeping all the others fixed. However --- and besides the non-convexity of the loss --- the feasible set formed by these $O(K^2)$ bi-linear constraints is not convex w.r.t $\{\A_k\}_k$. Moreover, this iterative procedure is computationally expensive as it requires solving multiple instances of constrained projected gradient descent and the number of necessary iterations to reach convergence is large in practice. All these issues make solving this  problem  challenging and computationally intractable even for reasonable values of $K$ and $n$. In what follows, we investigate a  workaround that   optimizes these matrices while guaranteeing their  orthogonality  and stochasticity as a part  of optimization.\\

\noindent {\bf Orthogonality.} Let $\exp(\gamma \hat{\A}_{k}) \oslash (\sum_{r=1}^K \exp(\gamma \hat{\A}_{r}))$ be a softmax reparametrization of $\A_{k}$, with  $\oslash$ being the entrywise hadamard division and  $\{\hat{\A}_k\}_k$ free parameters in $\mathbb{R}^{n \times n}$, it becomes possible to implement orthogonality by choosing large values of  $\gamma$ to make this softmax {\it crisp}; i.e., only one entry $\A_{kij}\gg 0$ while all others $\{\A_{k'ij}\}_{k'\neq k}$ vanishing thereby leading to $\A_k\odot \A_{k'}= {\bf 0}_{n\times n}$,  $\forall   k, k' \neq k$. By plugging this {\it crispmax} reparametrization into Eq.~\ref{matrixform3}, the gradient of the loss $\J$ (now w.r.t   $\{\hat{\A}_k\}_k$)  is updated using the chain rule as
\begin{equation}\label{eq00001111}
  \begin{array}{lll}
\displaystyle      & \displaystyle \frac{\partial \J}{\partial \v(\{\hat{\A}_k\}_k)} &= \displaystyle  \JJ_\textrm{orth} .   \frac{\partial \J}{\partial \v({\{{\A}_k\}_k})},
 \end{array}
                                                                                                                                                                                           \end{equation} 
being $\v(\{\A_k\}_k)$ a vectorization of $\{\A_k\}_k$ and $({\bf i},{\bf j})=(kij,k'i'j')$ an entry of the Jacobian $\JJ_{\textrm{orth}}$ as  
\begin{equation}\label{eq00001112}
  \begin{array}{lll}
   &\displaystyle  \left\{ \begin{array}{ll} \gamma {\A}_{kij}.(1 - \A_{k ij}) &  {\footnotesize \textrm{if} \ k=k', i=i', j=j'} \\ -\gamma {\A}_{kij}.\A_{k' ij} &  {\small \textrm{if} \ k \neq k', i=i', j=j'}   \\ 0 &  \textrm{\small  otherwise,}  \end{array}\right. 
                                                                                                                                                                                         \end{array}
                                                                                                                                                                                           \end{equation} 
                                                                                                                                                                                           \noindent here $\frac{\partial \J}{\partial \v({\{{\A}_k\}_k})}$   is obtained from layerwise gradient backpropagation. However, with this reparametrization, large values of  $\gamma$  may lead to numerical instability when evaluating the exponential. We circumvent this by choosing $\gamma$ that satisfies $\epsilon$-orthogonality: a surrogate property defined subsequently. \\
                                                                                                                                                                                           
\begin{definition}[$\epsilon$-orthogonality]  A basis  $\{\A_k\}_k$ is $\epsilon$-orthogonal if $\A_k\odot \A_{k'} \leq \epsilon  \ \mathds{1}_{n\times n}$, $\forall k,k' \neq k$,  with $\mathds{1}_{n \times n}$ being the $n \times n$ unitary matrix.
\end{definition}

Considering the above definition, (nonzero) matrices belonging to an $\epsilon$-orthogonal basis are linearly independent w.r.t $\langle .,. \rangle_F$ (provided that $\gamma$ is sufficiently large) and  hence this basis  is also minimal. The following proposition provides a tight lower bound on $\gamma$ that satisfies $\epsilon$-orthogonality. \\

\begin{proposition} [$\epsilon$-orthogonality bound] Consider $\{\A_{kij}\}_{ij}$ as the entries of the crispmax reparametrized  matrix  $\A_k$ defined as  $\exp(\gamma \hat{\A}_{k}) \oslash \big(\sum_{r=1}^K \exp(\gamma \hat{\A}_{r})\big)$.  Provided that  $\exists \delta>0:$ $\forall  i,j,\ell'$, $\exists !\ell$,  $\hat{\A}_{\ell ij} \geq  \hat{\A}_{\ell' ij}+\delta$ (with $\ell' \neq \ell$) and if $\gamma$ is at least {$$\displaystyle \frac{1}{\delta} \ln\bigg(\frac{K \sqrt{(1-2\epsilon)}}{1-\sqrt{(1-2\epsilon)}}+1\bigg)$$}
then  $\{\A_1,\dots,\A_K\}$ is $\epsilon$-orthogonal.  
\end{proposition}
\begin{proof}  
For any entry $i,j$, one may find $\ell$, $\ell'$ in $\{1,\dots,K\}$ (with $\ell \neq \ell'$) s.t.  $(\A_k \odot \A_{k'})_{ij} $ 
{\hspace*{-1cm} \begin{equation*}
\begin{array}{lll}
\displaystyle  &\leq &  \displaystyle  (\A_\ell \odot \A_{\ell'})_{ij}  \\ 
& & \\
&= &   \frac{1}{2} (\A_{\ell ij}^2 + \A_{\ell' ij}^2) -   \frac{1}{2} (\A_{\ell ij} -\A_{\ell' ij})^2     \\ 
& & \\
\displaystyle    &\leq&     \frac{1}{2}-   \frac{1}{2} (\A_{\ell ij} -\A_{\ell' ij})^2     \\ 
\displaystyle   &= &  \frac{1}{2}-  \frac{1}{2} \bigg( \displaystyle\frac{\exp(\gamma \hat{\A}_{\ell ij})-\exp(\gamma \hat{\A}_{\ell' ij})}{\exp(\gamma \hat{\A}_{\ell ij})+\exp(\gamma \hat{\A}_{\ell' ij})+\sum_{r=3}^K \exp(\gamma \hat{\A}_{rij})}\bigg)^2  \\ 
& &\\
\displaystyle &\leq &   \frac{1}{2} - \frac{1}{2}\bigg(\displaystyle\frac{\exp(\gamma \hat{\A}_{\ell ij}) -\exp(\gamma \hat{\A}_{\ell' ij})}{\exp(\gamma \hat{\A}_{\ell ij}) +(K-1)\exp(\gamma \hat{\A}_{\ell' ij})}\bigg)^2  \\
& & \\
\displaystyle & \leq  &   \frac{1}{2}  -  \frac{1}{2} \bigg(\displaystyle \frac{1}{1+ \frac{K}{\exp(\gamma \delta )-1}}\bigg)^2.
 \end{array}
\end{equation*}
The sufficient condition is to choose $\gamma$ such as  
\begin{equation*}
 \frac{1}{2}  -  \frac{1}{2} \bigg[\displaystyle \frac{1}{1+ \frac{K}{\exp(\gamma \delta )-1}}\bigg]^2 \leq \epsilon \implies \displaystyle \gamma \geq\displaystyle \frac{1}{\delta} \ln\bigg(\frac{K \sqrt{(1-2\epsilon)}}{1-\sqrt{(1-2\epsilon)}}+1\bigg).
 \end{equation*} 
 }
\begin{flushright}
$\blacksquare$
\end{flushright}

\end{proof} 
Following the above  proposition, setting $\gamma$ to the above lower bound guarantees $\epsilon$-orthogonality; for instance, when $K=2$, $\delta=0.01$ and provided that  $\gamma \geq 530$, one may obtain  $0.01$-orthogonality which is almost a strict orthogonality. This property is satisfied as long as one slightly disrupts the entries of  $\{\hat{\A}_k\}_k$ with random noise during training\footnote{\scriptsize whatever  the range of entries in these matrices $\{\hat{\A}_k\}_k$.}. However, this may still lead to another limitation; precisely,  bad local minima are observed  due to an {\it early} convergence to crisp adjacency matrices. We prevent this by steadily annealing the temperature $1/\gamma$  of the softmax through training epochs (using $\frac{\gamma.\textrm{epoch}}{\textrm{max\_epochs}}$ instead of $\gamma$) in order to make optimization focusing first on the loss, and then as optimization evolves, temperature cools down and allows reaching the aforementioned lower bound (thereby crispmax) and $\epsilon$-orthogonality at convergence. \\

\noindent {\bf Lightweight connectivity with stochasticity.} Unless explicitly mentioned, $\A_k$ is simply rewritten as $\A$. We consider  a reparametrization   $\A = h(\hat{\A}) \D( h(\hat{\A}^\top))^{-1}$, with $\D(.)$ being the degree matrix operator, $h$ a strictly monotonic positive function and this allows a free setting of the matrix $\hat{\A}$ during optimization while guaranteeing stochasticity. In practice, $h$ is set to $\exp$ and the original gradient is obtained, similarly to Eq.~\ref{eq00001111}, from layerwise gradient back propagation by multiplying the original gradient by the Jacobian  $[\JJ_{\textrm{stc}}]_{ij,i'j'}= [{\A}_{i'j'}.(\delta_{ii'} - \A_{ij})]$ with $\delta_{ii'}=1_{\{i=i'\}}$. Note that stochasticity, when combined with orthogonality, leightens connectivity by a factor $n$ compared to orthogonality whose factor does not exceed $K$; this combination is obtained by multiplying the underlying Jacobians, so the final gradient becomes  
                            \begin{equation}\label{eq0000}
\displaystyle \frac{\partial \J}{\partial \v(\{\hat{\A}_k\}_k)} = \displaystyle   \JJ_\textrm{stc}. \JJ_\textrm{orth}. \frac{\partial \J}{\partial \v(\{{\A}_k\}_k)},
                                                                                                                                            \end{equation} 
and this order of application is strict, as orthogonality sustains after stochasticity  while the converse is not necessarily guaranteed at the end of the optimization process.

\section{Experiments}
{\bf Database and settings.} We evaluate the performance of our GCN on the task of action recognition using the First-Person Hand Action (FPHA) dataset~\cite{garcia2018}. The latter includes 1175 skeletons belonging to 45 action categories which are performed by 6 different individuals in 3 scenarios. Action categories are highly variable with inter and intra subject variability including style, speed, scale and viewpoint. Each video (sequence of skeletons) is initially described with a handcrafted graph $\G = (\V,\E)$ where each node $v_j \in \V$ corresponds to the $j$-th hand-joint trajectory (denoted as $\{\hat{p}_j^t\}_t$)  and an edge $(v_j, v_i) \in  \E$ exists iff the $j$-th and the $i$-th trajectories are spatially connected. Each trajectory in $\G$ is processed using {\it temporal chunking}: first, the total duration of a  sequence is split into $M$ equally-sized temporal chunks ($M=4$ in practice), then the trajectory  coordinates  $\{\hat{p}_j^t\}_t$  are assigned to the $M$ chunks (depending on their time stamps) prior to concatenate the averages of these chunks; this produces the raw description of $v_j$, again denoted as $\psi(v_j)$. \\
\noindent {\bf Implementation details.} We trained the GCNs end-to-end using the Adam optimizer for 2,800 epochs  with a batch size equal to $600$, a momentum of $0.9$ and a global learning rate (denoted as $\nu(t)$)  inversely proportional to the speed of change of the cross entropy loss used to train our networks; when this speed increases (resp. decreases),   $\nu(t)$  decreases as $\nu(t) \leftarrow \nu(t-1) \times 0.99$ (resp. increases as $\nu(t) \leftarrow \nu(t-1) \slash 0.99$). In all these experiments, we use a GeForce GTX 1070 GPU device (with 8 GB memory), we evaluate the performances using the 1:1 setting proposed in~\cite{garcia2018} with 600 action sequences for training and 575 for testing, and we report the average accuracy over all the classes of actions.\\

\noindent {\bf Performances and comparison.} We compare the performances of our GCN  design against two baselines: handcrafted and learned. In the first baseline (known as power map), all the matrices  $\{ \A_k\}_{k}$  are evaluated upon the adjacency matrix $\A$ (taken from the input skeletons) as $\A_k = \A^{(k)}$ with $\A^{(k)} = \A^{(k-1)} \A$, $\A^{(0)} = \I$  and this defines nested supports for convolutions while  in the second baseline, all the adjacency matrices $\{ \A_k\}_{k}$ are learned using the objective function~(\ref{matrixform3}) but w/o orthogonality and stochasticity constraints. Table~\ref{table21} shows a comparison with these baselines and an ablation study of our complete model and the impact of orthogonality (separately and combined) on the performances. These results show that orthogonality has a clear and a consistent positive impact on the performances  while stochasticity (when combined with orthogonality) provides lightweight GCNs with an extra gain in accuracy. Clearly, these two constraints act as regularizers that also reduce the number of training parameters thereby leading to highly effective and also efficient GCNs. In order to further investigate the impact of these two constraints, we compare the underlying GCNs against lightweight ones obtained differently, with magnitude pruning; the latter consists first in zeroing the smallest parameters in the learned GCNs, and then fine-tuning the remaining parameters. As shown in table~\ref{table21},  lightweight GCNs, trained with orthogonality and stochasticity, clearly outperform those obtained with magnitude pruning+fine-tuning. Finally, we compare the classification performances of our GCN against other related methods in action recognition ranging from sequence based such as LSTM to deep graph (non-vectorial) methods, etc. (see table~\ref{compare2} and references within). From the results in these tables, our GCN brings a noticeable gain w.r.t. related state of the art methods.

 \begin{table}[ht]
 \begin{center}
\resizebox{0.69\columnwidth}{!}{
\begin{tabular}{cc||c|c||c|c||c|c}
  &  & \rotatebox{55}{H} & \rotatebox{55}{L} &     \rotatebox{55}{L+orth} & \rotatebox{55}{L+MP} &  \rotatebox{55}{L+orth+stc} & \rotatebox{55}{L+MP} \\
 \hline
  \hline
  \multirow{2}{*}{\rotatebox{35}{$K=2$}}& Accuracy (\%) &    84.17 & 83.30  &  84.52     & 84.52   & 83.65    & 81.56 \\
 &   Pruning rate (\%) &   none       &   none         &     50        &  50    &   95        & 95\\
  \hline 
   \multirow{2}{*}{\rotatebox{35}{$K=4$}} & Accuracy (\%) &   82.95 & 83.82  &  85.21 &  83.13  & 85.73 & 82.95\\
  & Pruning rate (\%) &   none        &   none         & 75         & 75     &  95           & 95  \\
  \hline 
  \multirow{2}{*}{\rotatebox{35}{$K=8$}}& Accuracy (\%) &  72.69 &  83.82  &  85.04  & 83.65 & \bf86.78 & 84.00 \\
  & Pruning rate (\%) &    none        &   none         &   87          &  87    &    95        & 95\\
  \end{tabular}}
\end{center}
\caption{Detailed performances, for different $K$, using handcrafted and learned connectivity w/o and with our constraints. We also compare these results with those of GCNs obtained using magnitude pruning (for the same pruning rates: $\floor{(1-\frac{1}{K}) \times 100}$ for L+orth vs. L+MP and $\floor{(1-\frac{1}{n})\times 100}$ for L+orth+stc vs. L+MP), here  H, L, orth, stc and MP stands respectively for handcrafted, learned, orthogonality, stochasticity and magnitude pruning.}\label{table21}
\end{table}

\begin{table}[ht]
\begin{center}
  \resizebox{0.59\columnwidth}{!}{
\begin{tabular}{cccc|c}
{\bf Method} & {\bf Color} & {\bf Depth} & {\bf Pose} & {\bf Accuracy (\%)}\\
\hline
  Two stream-color \cite{refref10}   & \cmark  &  \xmark  & \xmark  &  61.56 \\
Two stream-flow \cite{refref10}     & \cmark  &  \xmark  & \xmark  &  69.91 \\ 
Two stream-all \cite{refref10}      & \cmark  & \xmark   & \xmark  &  75.30 \\
\hline 
HOG2-depth \cite{refref39}        & \xmark  & \cmark   & \xmark  &  59.83 \\    
HOG2-depth+pose \cite{refref39}   & \xmark  & \cmark   & \cmark  &  66.78 \\ 
HON4D \cite{refref40}               & \xmark  & \cmark   & \xmark  &  70.61 \\ 
Novel View \cite{refref41}          & \xmark  & \cmark   & \xmark  &  69.21  \\ 
\hline
1-layer LSTM \cite{Zhua2016}        & \xmark  & \xmark   & \cmark  &  78.73 \\
2-layer LSTM \cite{Zhua2016}        & \xmark  & \xmark   & \cmark  &  80.14 \\ 
\hline 
Moving Pose \cite{refref59}         & \xmark  & \xmark   & \cmark  &  56.34 \\ 
Lie Group \cite{Vemulapalli2014}    & \xmark  & \xmark   & \cmark  &  82.69 \\ 
HBRNN \cite{Du2015}                & \xmark  & \xmark   & \cmark  &  77.40 \\ 
Gram Matrix \cite{refref61}         & \xmark  & \xmark   & \cmark  &  85.39 \\ 
TF    \cite{refref11}               & \xmark  & \xmark   & \cmark  &  80.69 \\  
\hline 
JOULE-color \cite{refref18}         & \cmark  & \xmark   & \xmark  &  66.78 \\ 
JOULE-depth \cite{refref18}         & \xmark  & \cmark   & \xmark  &  60.17 \\ 
JOULE-pose \cite{refref18}         & \xmark  & \xmark   & \cmark  &  74.60 \\ 
JOULE-all \cite{refref18}           & \cmark  & \cmark   & \cmark  &  78.78 \\ 
\hline 
Huang et al. \cite{Huangcc2017}     & \xmark  & \xmark   & \cmark  &  84.35 \\ 
Huang et al. \cite{ref23}           & \xmark  & \xmark   & \cmark  &  77.57 \\  

  \hline
Our best (table~\ref{table21})                    & \xmark  & \xmark   & \cmark  & \bf86.78                                                 
\end{tabular}}
\caption{Comparison against  state of the art methods.}\label{compare2}
\end{center}
 \end{table} 

\section{Conclusion} 
We introduce in this paper a novel framework that designs graph topology as a part of an ``end-to-end'' GCN learning. This topology is captured using multiple adjacency matrices whose optimization is constrained with orthogonality and stochasticity. The former makes it possible to remove the redundancy while the latter allows learning lightweight and   highly effective GCNs. These two constraints also act as regularizers that model structural relationships between network parameters in order to enhance both their  generalization and lightweightness. Experiments conducted on the challenging task of skeleton-based hand-gesture   recognition, shows the outperformance of the proposed lightweight GCNs against different baselines as well as the related work.

\end{document}